%% file: main.tex
\documentclass[reqno]{amsart}
\usepackage{ewald}
\numberwithin{equation}{section}
\usepackage{subcaption}
\newcommand\pen[1]{(#1){}^{+}} 

\renewcommand\i[2]{\iota(#1,#2)}      

\def\pyr{\mathcal{S}}

\let\epsilon\varepsilon

\title[Explicit neural network classifiers for 
non-separable data]
{Explicit neural network classifiers for 
non-separable data}

\author{Patr\'{i}cia Mu\~{n}oz Ewald}
\address[P. M. Ewald]{Department of Mathematics, University of Texas at Austin,
Austin TX 78712, USA} 
\email{ewald@utexas.edu}

\definecolor{burntorange}{HTML}{BF5700}
\definecolor{UTblue}{HTML}{00A9B7}
\definecolor{bluebonnet}{HTML}{005F86}
\hypersetup{
    pdfauthor={Patricia Munoz Ewald}, 
    pdftitle={Explicit neural network classifiers for non-separable data},
    linktoc=all,     
    colorlinks=true, 
    linkcolor=bluebonnet, 
    citecolor=burntorange,
    urlcolor=bluebonnet, 
}

\begin{document}

\maketitle

\begin{abstract}
    We fully characterize a large class of feedforward neural networks in terms of
    truncation maps. As an application, 
    we show how a ReLU neural network can implement a feature map which separates
    concentric data. 
\end{abstract}


\section{Introduction}

As deep learning models become more powerful and widespread, enhancing their efficiency
and interpretability becomes increasingly important. Neural networks, with their billions
of parameters trained primarily through gradient descent, often function as black-box
models. This opacity makes it challenging to understand how they make decisions. Research
on the mathematical foundations of deep learning aims to shed light on these processes,
helping to interpret and explain the decision-making mechanisms of neural
networks.

In previous work (joint with T. Chen) \cite{chenewald23deep, chenewald24hyperplanes}, we
showed that a feedforward neural network 
\begin{align}
    f_{\underline{\theta}}: 
    \mathbb{R}^{d_0} \stackrel{\underline{\theta}_{1}}{\longrightarrow} \,\,
    \mathbb{R}^{d_1} \stackrel{\underline{\theta}_{2}}{\longrightarrow} \,\,
    \cdots \stackrel{\underline{\theta}_{L-1}}{\longrightarrow} \mathbb{R}^{d_{L-1}}
    \stackrel{\underline{\theta}_{L}}{\longrightarrow} 
    \mathbb{R}^{d_{L}}
\end{align}
can be rewritten 
in terms of transformations $\tau$ of input space, or truncation maps,
\begin{align}
    f_{\underline{\theta}}:
    \mathbb{R}^{d_0}
    \stackrel{\tau^{(1)}}{\longrightarrow} \,\, \mathbb{R}^{d_0}
    \stackrel{\tau^{(2)}}{\longrightarrow} \cdots 
    \stackrel{\tau^{(L-1)}}{\longrightarrow} \mathbb{R}^{d_0}
    \longrightarrow
    \mathbb{R}^{d_{L}}.
\end{align}
This characterization makes it possible to observe how data transforms as it moves along the layers
of a neural network.
In the case of ReLU activation function, a description of the truncation map in terms of
cones yielded a geometric construction of networks which interpolate data for
multiclass classification tasks, with greatly reduced number of parameters.

In this work, we continue to develop the truncation map as a tool for investigating the
mathematical properties of neural networks. The main technical limitations of previous work were
that the width was bounded by the dimension of the input space, and that non-linear
decision boundaries for binary classification could not be achieved by the
characterization of the truncation map for ReLU. 
Our contributions are as follows:
\begin{itemize}
    \item We extend the description of feedforward neural networks in terms of truncation
        maps to the case of width $>$ input dimension, and completely characterize 
        truncation maps for ReLU activation (Section \ref{sec:tools}). 
    \item We apply these results to explicitly construct networks which interpolate data in
        2 model cases 
        (see Figure \ref{fig:cones}) 
        for binary classification data which is not linearly separable (Section
        \ref{sec:main}).
\end{itemize}

The result for concentric data (Figure \ref{fig:concentric}) 
is particularly interesting:
It realizes the first layer of a ReLU
neural network as a piecewise linear version of a feature map $\phi(x) = (x,
\norm{x}^{2})$ that a support vector machine might implement \cite{boseretal92}.
Moreover, it
illustrates the fact that  
ReLU neural networks with width at least (input dimension)$ + 1$ are universal
for real functions \cite{hanin19universal, haninsellke17}.
This distribution of data is also one of the examples considered in 
\cite{johnson18}, which proved that functions with a level set containing
a bounded component cannot be approximated by ReLU neural networks with width bounded
above by the input dimension.

\subsection*{Related work}
We make a few remarks on related work beyond what was mentioned in the previous paragraph. 
First, there are other works with explicit constructions of
ReLU neural networks. 
Notably,
\cite{cooper21} gave an explicit construction of a shallow network
which interpolates the data. 
For works characterizing critical points of loss
functions for ReLU networks with fixed architecture, we refer the reader to the
introduction of \cite{chenewald24hyperplanes}.

Second, 
we note that the reinterpretation of feedforward ReLU neural networks in terms of
truncation maps was independently discovered and studied in \cite{vallinetal23}. We adopt
the nomenclature
``polyhedral cones'' to match theirs, but otherwise match notation with our previous work
\cite{chenewald23deep, chenewald24hyperplanes}.

Finally, we observe that the maps $H_{j}$ defined in \cite{haninsellke17} are
truncation maps, but the affine maps considered are very restricted, and the focus of
that work was not on the geometric structure of those maps.

\begin{figure}
    \centering
    \begin{subfigure}{0.45\textwidth}
        \def\svgwidth{1\columnwidth}
        \raisebox{15pt}{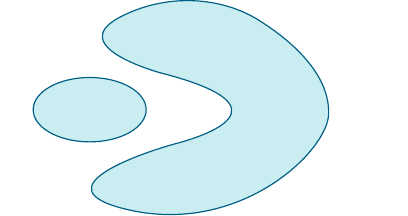}
        \caption{}
        \label{fig:crescent}
    \end{subfigure}
    \hspace{1em}
    \begin{subfigure}{0.3\textwidth}
        \def\svgwidth{1\columnwidth}
        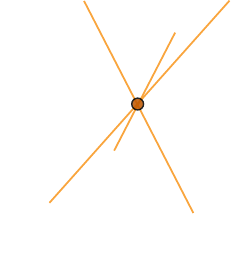
        \caption{}
        \label{fig:concentric}
    \end{subfigure}
    \caption{Simplified representations of the models of data distribution we consider
        here, and the polyhedral cones constructed such that the truncated data becomes
        linearly separable. Figure \ref{fig:crescent} corresponds to Proposition
        \ref{prop-crescent}, and Figure \ref{fig:concentric} corresponds to Corollary
        \ref{prop:kernel-trick}.}
    \label{fig:cones}
\end{figure}


\section{Setting and technical tools}
\label{sec:tools}

Consider neural networks of the form
\begin{align} \label{neuralnet}
    x^{(0)} &= x_{0} \in \mathbb{R}^{d_0} \,\, \text{ initial input,} \nonumber \\ 
    x^{(\ell )} &= \sigma(W_{\ell } x^{(\ell -1)} + b_{\ell }) \in \mathbb{R}^{d_{\ell}},
    \text{ for } \ell =1, \cdots, L-1, \\ 
    x^{(L)} &= W_{L}x^{(L-1)}+b_{L} \in \mathbb{R}^{d_{L}}, \nonumber
\end{align}
where $W_{\ell} \in \mathbb{R}^{d_{\ell }\times d_{\ell -1}}$ for $\ell =1, \cdots, L$
are weight matrices, $b_{\ell } \in \mathbb{R}^{d_{\ell }}$ for $\ell=1, \cdots, L$ are
bias vectors, and $\sigma$ is an activation function which acts component-wise on vectors.
In the present work, we will focus on the ramp function (ReLU),
\begin{align}
    \sigma(x)_{i} := (x_{i})_{+} = \max(x_{i},0).
\end{align}

Given $W\in \mathbb{R}^{m \times n}$ and $ b\in \mathbb{R}^{m}$, consider the
family of truncation maps
\begin{align}
     \tau_{W,b}:\mathbb{R}^{n} &\to \mathbb{R}^{n} \nonumber \\
     x &\mapsto \pen{W}\left(\sigma (Wx + b) - b\right),
\end{align}
where $\pen{W}$ is the generalized inverse of $W$.
Define the cumulative parameters
\begin{alignat}{2}
    \label{W(ell)}
    W^{(1)} &:= W_1, \quad W^{(\ell )} &&:= W_{\ell } \cdots W_1 = W_{\ell } W^{(\ell
    -1)},  \nonumber \\ 
        b^{(1)} &:= b_{1}, \quad\quad b^{(\ell )} &&:= W_{\ell} b^{(\ell -1)} + b_{\ell},
\end{alignat}
associated to a network \eqref{neuralnet},
and for any input $x = x^{(0)}$, let $x^{(\tau, \ell )}$ correspond to the output of
$\ell$ many truncation maps applied to $x$, 
\begin{align}
    x^{(\tau, 0)} := x^{(0)}, \quad
    x^{(\tau,\ell )}
        &:= \tau_{W^{(\ell)},b^{(\ell)}} \, (x^{(\tau ,\ell -1)}), \quad \ell=1, \cdots,
        L-1.
\end{align}

In \cite{chenewald23deep, chenewald24hyperplanes}, the author and a collaborator showed
that a neural network can be described in terms of truncation maps and cumulative
parameters:
if all weight matrices $W_{\ell}$ are surjective, $\ell =1,\cdots, L$, then
\begin{align}
    \label{taurecursive}
    x^{(\ell)} &= W^{(\ell)}x^{(\tau, \,\ell)} + b^{(\ell)},
\end{align}
for $\ell =1, \cdots, L-1$, and
\begin{align}
    \label{taulast}
    x^{(L)} &= W^{(L)} x^{(\tau,\, L-1)} + b^{(L)}.
\end{align}

\vspace{1em}
\subsection{Polyhedral cones}
In \cite{chenewald24hyperplanes}, the action
of a truncation map $\tau$ for a ReLU neural network was given in terms of cones in input
space, and $\tau$ acts either as the identity or as a map to a single point on two
distinguished regions in input space. 
This description was used to construct zero loss classifiers for linearly separable data.
For data that is not linearly separable, this is not sufficient.


The following lemma completely describes the action of the truncation map with ReLU
activation function for surjective weight matrices. 

\begin{lemma}
    \label{new-cone-lemma}
    Given $W:\mathbb{R}^{n}\to \mathbb{R}^{m}$ a surjective matrix and $b\in
    \mathbb{R}^{m}$, there exist $p \in \mathbb{R}^{n}$ and $(v_{i})_{i=1}^{m} \subset
    \mathbb{R}^{n}$ such that,
    for 
    \begin{align} 
        \label{xdecomp}
        x = p + \tilde{x}  + \sum_{i=1}^{m} a_{i} v_{i}, 
    \end{align} 
    where $\tilde{x} \in \ker W$, then
    \begin{align}
        \label{tau}
        \tau_{W,b}(x) &= \pen{W} W p + \sum_{i; a_{i} > 0} a_{i} v_{i} \nonumber \\
                      &= \pen{W}W x - \sum_{j;a_{j} \leq 0} a_{j} v_{j}. 
    \end{align}
    Conversely, given $p \in \mathbb{R}^{n}$ and a linearly independent set
    $(v_{i})_{i=1}^{m} \subset \mathbb{R}^{n}$, there exist $W \in \mathbb{R}^{m\times n}$
    and $b\in \mathbb{R}^{m}$ that realize \eqref{tau} for $x \in \mathbb{R}^{n}$
    decomposed as
    \begin{align} 
        \label{xdecomp2}
        x = p + \tilde{v}  + \sum_{i=1}^{m} a_{i} v_{i}, 
    \end{align} 
    where $\tilde{v} \in (span(v_{i})_{i=1}^{m})^{\perp}$.
\end{lemma}

\begin{proof}
Let $W:\mathbb{R}^{n} \to \mathbb{R}^{m}$ be a surjective linear map.
Recall
that $W\pen{W} = \Id_{m\times m}$ and  $\pen{W}W$ is the orthogonal projector to $(\ker
W)^{\perp}$.  Define
\begin{align}
    v_{i} := \pen{W}(e_{i}^{m}), \,\, i=1, \cdots, m,
\end{align}
and note that $W(v_{i}) = W\pen{W} e_{i}^{m} = e_{i}^{m}$, and $\pen{W}W v_{i} = 
\pen{W}e_{i}^{m} = v_{i}$.  Then fixing $p\in \mathbb{R}^{n}$, for any $x\in \mathbb{R}^{n}$
we have the decomposition
\begin{align} 
    x = p + \tilde{x}  + \sum_{i=1}^{m} a_{i} v_{i}, 
\end{align} 
where $\tilde{x} \in \ker W$. Let $b = -Wp$. Then
\begin{align}
    \tau_{W,b}(x) &= \pen{W} \left( \sigma (Wx + b) - b \right) \nonumber \\
                  &= \pen{W} \left(\sigma \left( Wp + W\tilde{x} + \sum_{i=1}^{m} a_{i}
                  W(v_{i}) - Wp \right) - b\right) \nonumber \\
                  &= \pen{W} \left( \sigma \left(\sum_{i=1}^{m} a_{i} e_{i} \right) + Wp
                  \right) \\
                  &= \pen{W} W p + \sum_{i; a_{i} > 0} a_{i} v_{i} \nonumber \\
                  &= \pen{W}W x - \sum_{j;a_{j} \leq 0} a_{j} v_{j}. \nonumber
\end{align}

Note that given $W, b$, where $W$ is surjective, we can find $p := -\pen{W}b$ and $v_{i}
:= \pen{W} e_{i}^{m} \in \mathbb{R}^{n}$, for $i=1, \cdots, m$. 
Conversely, given $p \in \mathbb{R}^{n}$ and $(v_{i})_{i=1}^{m} \subset \mathbb{R}^{n}$
linearly independent, we can define $\pen{W} :=
[ v_1 \cdots v_{n}]$, which is injective and so $W$ is surjective, and $b := - Wp$.
\end{proof}

\begin{coro}
    \label{coro:proj}
    For fixed surjective $W \in \mathbb{R}^{m\times n}$ and $b \in \mathbb{R}^{m}$, the
    ReLU truncation map $\tau_{W,b}$ is idempotent: $\tau_{W,b} \circ \tau_{W,b} =
    \tau_{W,b}.$ 
\end{coro}

\begin{proof}
    For 
    \begin{align}
        x = p + \tilde{x}  + \sum_{i=1}^{m} a_{i} v_{i}
    \end{align}
    as in \eqref{xdecomp}, we have
    \begin{align}
        \tau_{W,b}(x) &= \pen{W} W p + \sum_{i; a_{i} > 0} a_{i} v_{i} \nonumber \\ 
                      &= p - \mathcal{P}_{\ker W} (p) + \sum_{i; a_{i} > 0}
                      a_{i} v_{i} ,
    \end{align}
    since $\pen{W}W = \mathcal{P}_{(\ker W)^{\perp}} = \left(\Id - \mathcal{P}_{\ker
    W}\right)$. Applying \eqref{tau} again,
    \begin{align}
        \tau_{W,b}\circ \tau_{W,b}(x) = \pen{W} W p + \sum_{i; a_{i} > 0} a_{i} v_{i}.
    \end{align}
\end{proof}

\begin{remark}
    We see from \eqref{tau} that two points $x_1, x_2$ for which $\{i, a_{i}(x_1) > 0\}= \{i,
    a_{i}(x_2) > 0\} =: I$ and $a_{i}(x_1) = a_{i}(x_2)$ for all such $i\in I$ will be
    truncated to the same point.
\end{remark}

Consider $W \in GL(n)$ and $b\in \mathbb{R}^{n}$. Then $p \in \mathbb{R}^{n}$
and $\underline{v}:=(v_1, \cdots, v_{n})$ given by Lemma \ref{new-cone-lemma}
define the polyhedral cones 
\begin{align} 
    \label{S+}
    \pyr_{+}(p,\underline{v}) := \left\{p + \sum_{i=1}^{n} a_{i} v_{i} : a_{i} \geq 0,
    i=1, \cdots, n \right\} 
\end{align}
and 
\begin{align}
    \label{S-}
    \pyr_{-}(p,\underline{v}) := \left\{p +  \sum_{i=1}^{n} a_{i} v_{i} : a_{i} \leq  0,
    i=1, \cdots, n \right\},
\end{align}
and their union $\mathcal{S} := \pyr_{+} \cup \pyr_{-}$ (see Figure \ref{fig:poly}). We
will make a few remarks and definitions which will be useful later.

\begin{figure}[h]
    \centering
    \def\svgwidth{0.45\columnwidth}
    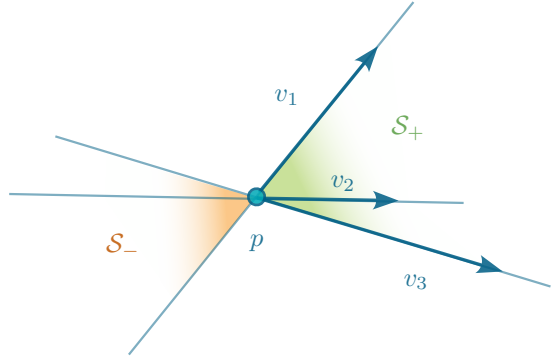
    \caption{For the ReLU activation function, how a truncation map $\tau $ acts on the data
    depends on the positioning of the data with respect to the associated polyhedral cone
    $\pyr = \pyr_{+}\cup \, \pyr_{-}$, characterized by a base point $p$ and edges given
by a set of linearly independent vectors $v_{i}$, $i=1, \cdots, m$ (in the picture,
$m=3$).}
    \label{fig:poly}
\end{figure}

First, we can restate Lemma \ref{new-cone-lemma} in different terms: 
If $x\in \pyr_{+}$ then $\tau _{W,b}(x) = x$, and if $x\in \pyr_{-}$ then
$\tau_{W,b}(x) = p.$ 
Moreover, $\tau_{W,b}(x) \in \pyr_{+}$ for all $x\in \mathbb{R}^{n}$,
and any point in the complement $ \pyr^{c}$ gets projected
to the boundary of $\pyr_{+}$. 
In this way, the ReLU truncation maps are easily seen to be (non-linear)
projections (Corollary \ref{coro:proj}).

Next, the boundary of 
$\pyr_{+} \subset \mathbb{R}^{n}$ can be decomposed into faces, similarly to a simplex,
characterized by $a_{k} = 0$, $k=1, \cdots, n$.
Each one of these $(n-1)$-faces spans a hyperplane in $\mathbb{R}^{n}$,
\begin{align} 
    \label{Hk}
    H_{k}(p, \underline{v}) := \left\{ p + \sum_{i=1, i \neq k}^{n} a_{i} v_{i}: a_{i}\in
        \mathbb{R} \right\}.
\end{align}

Finally, if $\underline{v}=(v_1, \cdots, v_{m})$ contains only $m<n$ linearly independent
vectors in $\mathbb{R}^{n}$, as may be the case if $W\in \mathbb{R}^{m \times n}$ in Lemma
\ref{new-cone-lemma}, then we let
\begin{align} 
    \pyr_{+}(p,\underline{v}) := \left\{p + \sum_{i=1}^{m} a_{i} v_{i} + \tilde{v}: a_{i}
    \geq 0, i=1, \cdots, m, \, \tilde{v} \in (span(v_{i})_{i=1}^{m})^{\perp}\right\},
\end{align}
and similarly for $\pyr_{-}$. In this case, $\mathcal{S}(p, \underline{v})$ can be thought
of as a family of polyhedral cones $\mathcal{S}(p+\tilde{v}, \underline{v})$ defined as
in \eqref{S+}
on $m$-dimensional affine subspaces
\begin{align}
    H_{\tilde{v}} = \left\{ \tilde{v} + \sum_{i=1}^{m} a_{i} v_{i}: a_{i} \in
    \mathbb{R}\right\} \subset \mathbb{R}^{n},
\end{align}
indexed by $\tilde{v} \in (span(v_{i})_{i=1}^{m})^{\perp}$, or $\tilde{v} \in \ker(W)$  in
Lemma \ref{new-cone-lemma}.

\vspace{1em}
\subsection{Increased width}
Next, we study what happens if $W$ is not surjective. We leave the case where
$W$ is not full-rank for future work, and consider here the situation where $W\in
\mathbb{R}^{m\times n}$ for $n < m$, and $W$ is
injective.

Define
\begin{align} 
    \i nm&:\mathbb{R}^{n} \to \mathbb{R}^{m},\quad n\leq m, \nonumber\\
    \i nm(e^{n}_{i}) &=
        e^{m}_{i}, \quad i = 1, \cdots, n,
\end{align}
where $\left\{e_{i}^{n}\right\}_{i=1}^{n}$ is the standard basis for $\mathbb{R}^{n}$,
so that $\i nm \in \mathbb{R}^{m \times n}$ is a block matrix with upper block
$\Id_{n\times n}$ and lower block made of zeros. 

\begin{lemma}
    \label{increase-lemma}
    Let $W:\mathbb{R}^{n} \to \mathbb{R}^{m}$ be a linear map, for $n<m$. There exists
    $\tilde{W} \in \mathbb{R}^{m\times m}$ such that $W = \tilde{W} \i n m$.
    If $W$ is injective, then $\tilde{W}$ can be made invertible. In that case,
    \begin{align}
        \label{tau-up}
        \sigma(Wx + b) = \tilde{W} \tau_{\tilde{W}, b} (\tilde{x}) + b,
    \end{align}
    for $\tilde{x}  = \i nm x \in \mathbb{R}^{m}$.
\end{lemma}

\begin{proof}
    First, we construct $\tilde{W}$.
    Suppose  $w_{i} \in \mathbb{R}^{m}$, $i=1, \cdots, n$, are the columns of $W$,
    \begin{align}
        W = \begin{bmatrix}
            w_1 & \cdots & w_{n}
        \end{bmatrix}.
    \end{align}
    Then let
    \begin{align}
        \tilde{W} := \begin{bmatrix}
            w_1 & \cdots & w_{n} & \tilde{w}_{1} & \cdots & \tilde{w}_{m-n} 
        \end{bmatrix},
    \end{align}
    so $\tilde{W}(e_{i}^{m}) = w_{i}$, for $i=1, \cdots, n$, and
    $\tilde{W}(e_{n+j}^{m}) = \tilde{w}_{j}$ for $j=1, \cdots, m-n$, for an arbitrary
    collection of vectors
    $\{\tilde{w}_{j}\}_{j=1}^{m-n} \subset \mathbb{R}^{m}$.
    We check
    \begin{align}
        \tilde{W} \i nm (e_{i}^{n}) = \tilde{W}(e_{i}^{m}) = w_{i} = W(e_{i}^{n}),
    \end{align}
    so $W = \tilde{W} \i nm$.
    Moreover, if $W$ is injective, then $\left\{\tilde{w}_{j}\right\}_{j=1}^{m-n}$ can be
    chosen so that the columns of $\tilde{W}$ are linearly independent, and $\tilde{W}\in
    GL(n).$

    Finally, we check \eqref{tau-up}:
    \begin{align}
        \tilde{W} \tau_{\tilde{W}, b} (\i nm x) + b &= \tilde{W} \pen{\tilde{W}}
        \left(\sigma(\tilde{W} \i nm x + b) - b \right) + b \nonumber \\
                                                 &= \sigma\left( Wx + b\right).
    \end{align}
\end{proof}

In other words, an increase in width can be split into a linear layer map which includes
into a higher-dimensional 
space $\iota: \mathbb{R}^{n} \hookrightarrow \mathbb{R}^{m}$, followed by a non-linear
map between layers with the same dimension.

We can now rewrite a neural network for which $d_0 < d_1$ and
$d_1 \geq d_2 \geq \cdots \geq d_{L}$ in terms of truncation maps:
Consider such a neural network,
\begin{align}
    \label{net1}
    \mathbb{R}^{d_0} \stackrel{\underline{\theta}_{1}}{\longrightarrow} \,\,
    \mathbb{R}^{d_1} \stackrel{\underline{\theta}_{2}}{\longrightarrow} \,\,
    \mathbb{R}^{d_2} \stackrel{\underline{\theta}_{3}}{\longrightarrow} \cdots
    \stackrel{\underline{\theta}_{L}}{\longrightarrow} \mathbb{R}^{d_{L}},
\end{align}
defined as in \eqref{neuralnet},
with $L$ layers and parameters $\underline{\theta} = (
\underline{\theta}_{1}, \cdots, \underline{\theta}_{L}$), where $\underline{\theta}_{\ell }
= (W_{\ell }, b_{\ell }) \in \mathbb{R}^{d_{\ell } \times d_{\ell-1}} \times
\mathbb{R}^{d_{\ell }}$ for $\ell  = 1, \cdots, L$. Assuming all weight matrices are
full-rank, and letting $\tilde{x} :=
\i{d_0}{d_1}(x)$, \eqref{tau-up} gives
\begin{align}
    x^{(1)} &= \sigma(W_1 x + b_1) = \sigma(\tilde{W}_{1}(\tilde{x}) + b_1) =
    \tilde{W}_{1} \tau_{\tilde{W}_{1}, b_1} (\tilde{x}) + b_1
\end{align}
for a suitable $\tilde{W}_{1} \in GL(d_1)$, which implies
\begin{align}
    x^{(2)} &= \sigma(W_2 x^{(1)} + b_2) \nonumber \\
            &= \sigma(W_2 \tilde{W}_{1} \tau_{\tilde{W}_{1}, b_1}(\tilde{x}) + W_{2}b_1 +
            b_2)  \nonumber \\
            &= \sigma (\tilde{W}^{(2)} \tilde{x}^{(\tau, 1)} + b^{(2)}) \\
            &= \tilde{W}^{(2)} \tilde{x}^{(\tau, 2)} + b^{(2)}, \nonumber
\end{align}
where we took $\tilde{W}^{(2)} = W_2 \tilde{W}_{1}$ instead of $W^{(2)} = W_2 W_1$ as in
the $d_0 \geq d_1$ case.
Then a modified version of \eqref{taulast} holds,
\begin{align}
    \label{newtaulast}
    x^{(L)} = \tilde{W}^{(L)} \tilde{x}^{(\tau , L-1)} + b^{(L)},
\end{align}
where 
\begin{align}
    \tilde{x}^{(\tau, L-1)} = \tau_{\tilde{W}^{(L-1)}, b^{(L-1)}} \circ \cdots \circ
    \tau_{\tilde{W}^{(1)}, b^{(1)}} (\, \tilde{x} \,),
\end{align}
for
\begin{align}
    \label{newW(ell)}
    \tilde{W}^{(\ell)} = W_{\ell } \cdots W_2 \tilde{W}_{1}, \quad \ell = 1, \cdots, L.
\end{align}
Finally, \eqref{newtaulast} implies that \eqref{net1} is equivalent to
\begin{alignat}{2}
    \mathbb{R}^{d_0} \stackrel{\iota}{\hookrightarrow}
    \mathbb{R}^{d_1} \stackrel{\tau_{\underline{\theta}^{(1)}}}{\longrightarrow} \,\,
     &\mathbb{R}^{d_1}
    \stackrel{\tau_{\underline{\theta}^{(2)}}}{\longrightarrow} \cdots
    \stackrel{\tau_{\underline{\theta}^{(L-1)}}}{\longrightarrow} \mathbb{R}^{d_1}
    \stackrel{\underline{\theta}^{(L)}}{\longrightarrow} \mathbb{R}^{d_{L}},
\end{alignat}
where $\underline{\theta}^{(\ell)} = (\tilde{W}^{(\ell)}, b^{(\ell )})$, for $\ell = 1,
\cdots, L$.

\begin{remark}
    While Lemma \ref{increase-lemma} could be applied at any given layer of a neural
    network, and (with extra work) could be used multiple times so that a version of
    \eqref{newtaulast} would hold for arbitrary feedforward networks, we restrict
    ourselves to the case $d_0 \leq d_1 \geq d_2 \geq \cdots \geq d_{L}$ in this paper
    because it is natural, and sufficient for dealing with the toy models considered here.
    Moreover, it encompasses the architectures needed to
    approximate continuous functions on a compact set
    \cite{haninsellke17}.
\end{remark}

\vspace{1em}
\section{Application to binary classification}
\label{sec:main}

Consider a data set $\mathcal{X} = \sum_{j=1}^{2} \mathcal{X}_{j}$ 
split into two disjoint classes.  
For $\mathcal{X} \subset
\mathbb{R}^{d_0}$, a classifying neural
network with $L$ layers is a map
\begin{align}
    f_{\underline{\theta}}: \mathbb{R}^{d_0} \to \mathbb{R}^{2}
\end{align}
parametrized by a vector of parameters $\underline{\theta}
= (
\underline{\theta}_{1}, \cdots, \underline{\theta}_{L}$), where $\underline{\theta}_{\ell }
= (W_{\ell }, b_{\ell }) \in \mathbb{R}^{d_{\ell } \times d_{\ell-1}} \times
\mathbb{R}^{d_{\ell }}$ for $\ell  = 1, \cdots, L$.

In the following results, we show how one truncation map can make
the data linearly separable for two toy models of data that is not linearly
separable. Then the results from
\cite{chenewald24hyperplanes} can be applied to construct a neural network which
interpolates the data.

We start with the case where the data can be separated by a single polyhedral cone, as in
Figure \ref{fig:crescent}.

\begin{prop} \label{prop-crescent}
    Suppose $\mathcal{X} = \mathcal{X}_{1} \cup \mathcal{X}_{2} \subset \mathbb{R}^{n}$,
    and there exist $p \in \mathbb{R}^{n}$ and
    $\{v_{i}\}_{i=1}^{m} \subset
    \mathbb{R}^{n}$ a set of linearly independent vectors, such that $\mathcal{X}_{1}
    \subset \pyr_{-}(p,\underline{v})$ and $\mathcal{X}_{2} \subset (\pyr_{-})^{c}$, for
    $n \geq m$, and $\underline{v} = (v_{i})_{i=1}^{m}$.

    Then there exist $W \in \mathbb{R}^{m \times n}$ and $b\in \mathbb{R}^{m}$ such
    that the truncated data $\tau_{W,b}(\mathcal{X})$ is linearly separable.
\end{prop}
\begin{proof}
    Let $W$ and $b$ be defined as in Lemma \ref{new-cone-lemma}.
    For all $x \in \mathcal{X}_{1}$  we have
    \begin{equation}
        \tau_{W,b}(x) = \pen{W}W p.
    \end{equation}
    For any $x\in \mathcal{X}_{2}$, written as 
    \begin{align}
        x = p + \tilde{x} + \sum_{i=1}^{m} a_{i}(x) v_{i},
    \end{align}
    for some $\tilde{x}\in \ker W$, there exists $i\in \left\{1, \cdots, n \right\}$
    such that $a_{i}(x) > 0$, and 
    \begin{equation}
        \tau_{W,b}(x) = \pen{W}Wp + \sum_{i;a_{i}(x)>0} a_{i}(x)\, v_{i}.
    \end{equation}
   
    We can now construct a hyperplane which separates the two classes. Let
    \begin{equation}
        \tilde{a} := \min\left\{ a_{i}(x): x \in \mathcal{X}_{2}, i=1, \cdots, m,
        \text{ and } a_{i}(x) > 0 \right\},
    \end{equation}
    and consider the $(m-1)$-dimensional affine subspace
    \begin{align} 
        \label{hyper}
        H(p, (v_{i})_{i=1}^{m}, \tilde{a}) := \left\{ \pen{W}W p + \frac{\tilde{a}}{2} v_1
        + \sum_{j=2}^{m} h_{j} (v_1 - v_{j}) : h_{j} \in \mathbb{R}, j=2, \cdots, m
        \right\} \subset \mathbb{R}^{n},
    \end{align}
    which we will denote simply by $H$ in the remainder of this proof.
    Note that $H, \tau_{W,b}(\mathcal{X}_{1})$ and $\tau_{W,b}(\mathcal{X}_{2}) \subset
    (\ker W)^{\perp}$.  Then for $x\in \mathcal{X}_{1}$,
    \begin{align}
        \tau_{W,b}(x) &= \pen{W}W p \nonumber \\ 
                      &= \left(\pen{W}W p + \frac{\tilde{a}}{2}v_1 \right) -
                      \frac{\tilde{a}}{2}v_1 ,
    \end{align}
    and for $x\in \mathcal{X}_{2}$,
    \begin{align}
        \tau_{W,b}(x) &= \pen{W}W p + \sum_{a_{i}(x) >0} a_{i}(x) v_{i} \nonumber \\
                      &= \pen{W}W p + \sum_{i=1}^{m} \sigma(a_{i}(x)) v_{i} \nonumber \\
                      &= \pen{W}W p + \frac{\tilde{a}}{2}v_1 + \left(\sigma(a_1(x)) -
                      \frac{\tilde{a}}{2}\right)v_1 + \sum_{i>1}\sigma(a_{i}(x)) v_{i} \\
                      &= \left( \pen{W}W p + \frac{\tilde{a}}{2}v_1 +
                          \sum_{i>1}\sigma(a_{i}(x)) (v_{i} - v_1) \right) +
                          \underbrace{\left(\sum_{i=1}^{m} \sigma(a_{i}(x)) -
                          \frac{\tilde{a}}{2}\right)}_{>0} v_1, \nonumber
    \end{align}
    from which we see that $\tau_{W,b}(\mathcal{X}_{1})$ and
    $\tau_{W,b}(\mathcal{X}_{2})$ are on opposite sides of the hyperplane $H$ inside
    $(\ker W)^{\perp}$. To extend $H$ to a hyperplane in $\mathbb{R}^{n}$ which separates
    the truncated classes, it suffices to take $H + \ker W$.
\end{proof}

In this case, the data can be interpolated by applying 
a second truncation map and an affine map,
\begin{align}
    \mathbb{R}^{d_0}
    \stackrel{\tau_{\underline{\theta}^{(1)}}}{\longrightarrow} \,\, \mathbb{R}^{d_0}
    \stackrel{\tau_{\underline{\theta}^{(2)}}}{\longrightarrow} \mathbb{R}^{d_0}
    \stackrel{\underline{\theta}^{(3)}}{\longrightarrow} \mathbb{R}^{2},
\end{align}
which is equivalent to a neural network with the following architecture:
\begin{align}
    \label{net-1}
    \mathbb{R}^{d_0} \stackrel{\underline{\theta}_{1}}{\longrightarrow}
    \mathbb{R}^{d_1} \stackrel{\underline{\theta}_{2}}{\longrightarrow}
    \mathbb{R}^{d_2} \stackrel{\underline{\theta}_{3}}{\longrightarrow} \mathbb{R}^{2},
\end{align}
where $d_0 \geq d_1 \geq d_2 \geq 2$.

\vspace{1ex}
Consider now data which is concentric, as in Figure \ref{fig:concentric}.
It is easy to see in the 2-dimensional case that no polyhedral cone in $\mathbb{R}^{2}$
will satisfy the assumptions of Proposition \ref{prop-crescent}. By adding one extra
dimension and using a cone in $\mathbb{R}^{3}$, the data can be projected in a way that
makes it linearly separable. 

\begin{coro} 
    \label{prop:kernel-trick}
    Suppose $\mathcal{X} = \mathcal{X}_{1} \cup \mathcal{X}_{2} \subset \mathbb{R}^{n}$,
    and there exists an $n$-simplex $S$ such that $\mathcal{X}_{1} \subset S$ and
    $\mathcal{X}_{2} \subset S^{c}$. Then 
    there exist $W \in GL(n+1)$ and $b\in \mathbb{R}^{n+1}$ such that 
    $\tau_{W,b} \circ \i{n}{n+1}(\mathcal{X})$ is linearly separable.
\end{coro}

\begin{proof}
    First, include the data into $\mathbb{R}^{n+1}$ with $\iota := \i{n}{n+1}$.
    Next, let $s_1, \cdots, s_{n+1}$ be the vertices for $S' = \iota(S)$, 
    and
    choose $p \in \mathbb{R}^{n+1}$ a positive distance away from the hyperplane in
    $\mathbb{R}^{n+1}$ spanned by $S'$. 
    Let 
    \begin{align}
        v_{i} := p - s_{i}, \quad i=1, \cdots, n+1.
    \end{align}
    Then a polyhedral cone $\mathcal{S}_{-}$ can be constructed from $p$ and
    $\left(v_{i}\right)_{i=1}^{n+1}$ for which the assumptions of Proposition
    \ref{prop-crescent} are satisfied: The simplex $S'$ is entirely contained in 
    $\mathcal{S}_{-}$ by construction, so that
    \begin{align}
        \iota(\mathcal{X}_{1}) \subset S' \subset \mathcal{S}_{-}.
    \end{align}
    Moreover, $\iota(\mathcal{X}_{2}) \subset (\pyr_{-})^{c}$, since
    $\iota(\mathcal{X}_{2}) \subset (S')^{c}\cap span(S')$ and $\pyr_{-}\cap span(S') =
    S'$.
    
    Proposition \ref{prop-crescent} now applies, and
    a hyperplane like \eqref{hyper} can be found that separates $\tau_{W,
    b}(\iota (\mathcal{X}_{1}))$ and $\tau_{W,b}(\iota(\mathcal{X}_{2}))$, for 
    $W \in GL(n+1)$ and $b\in \mathbb{R}^{n+1}$ given by 
    Lemma \ref{new-cone-lemma}.
\end{proof}

Putting together Lemma \ref{increase-lemma} and Corollary \ref{prop:kernel-trick}
gives our main result:

\begin{thm}
    \label{thm:main}
    Suppose $\mathcal{X} = \mathcal{X}_{1} \cup \mathcal{X}_{2} \subset \mathbb{R}^{d_0}$,
    and there exists a $d_0$-simplex $S$ such that $\mathcal{X}_{1} \subset S$ and
    $\mathcal{X}_{2} \subset S^{c}$. Moreover, let the classes be labeled by two linearly
    independent vectors $y_1, y_2 \in \mathbb{R}^{2}$.
    Then this data can be interpolated
    by a feedforward ReLU neural
    network \eqref{neuralnet} with $2$ hidden layers,
    \begin{align}
        \label{mainnet}
        \mathbb{R}^{d_0} \stackrel{\underline{\theta}_{1}}{\longrightarrow}
        \mathbb{R}^{d_1} \stackrel{\underline{\theta}_{2}}{\longrightarrow}
        \mathbb{R}^{d_2} \stackrel{\underline{\theta}_{3}}{\longrightarrow}
        \mathbb{R}^{2},
    \end{align}
    where $d_1 = d_0 + 1$ and $d_1 \geq d_2 \geq 2$. 
\end{thm}
\begin{proof}
    We will construct the following sequence of maps to interpolate the data:
    \begin{align}
        \label{truncationnet}
        \mathbb{R}^{d_0} \stackrel{\iota}{\hookrightarrow}
        \mathbb{R}^{d_0 + 1}
        \stackrel{\tau_{\underline{\theta}^{(1)}}}{\longrightarrow} \,\, \mathbb{R}^{d_0 +
        1} \stackrel{\tau_{\underline{\theta}^{(2)}}}{\longrightarrow} \mathbb{R}^{d_0 + 1}
        \stackrel{\underline{\theta}^{(3)}}{\longrightarrow} \mathbb{R}^{2}.
    \end{align}
    Corollary \ref{prop:kernel-trick} gives the inclusion map $\iota:= \i{d_0}{d_0+1}$,
    and the parameters
    $\tilde{W}_{1} \in GL(d_0+1)$ and $b^{(1)} \in \mathbb{R}^{d_0+1}$ for
    $\underline{\theta}^{(1)}=(\tilde{W}_{1},b^{(1)})$, such that
    $\tau_{\underline{\theta}^{(1)}}(\iota(\mathcal{X}))$ is linearly separable.  
    Then the cumulative parameters $\underline{\theta}^{(2)} = (W^{(2)}, \, b^{(2)})$
    defining the second truncation map and $\underline{\theta}^{(3)} = (W^{(3)}, \,
    b^{(3)})$ defining the affine map in \eqref{truncationnet} can be found using the
    hyperplane given by \eqref{hyper} for Corollary \ref{prop:kernel-trick} and
    the construction in \cite{chenewald24hyperplanes}. 

    Finally, 
    the parameters $(\underline{\theta}_{i})_{i=1}^{3}$ defining
    \eqref{mainnet} are given according to
    Lemma \ref{increase-lemma} and  
    \eqref{newW(ell)}: the weight matrices are
    \begin{align}
        W_1 &= \tilde{W}_{1} \i{d_0}{d_0+1}, \nonumber  \\
        W_{2} &= W^{(2)} \tilde{W}_{1}^{-1}, \\
        W_{3} &= W^{(3)} \pen{W^{(2)}}, \nonumber 
    \end{align}
    the bias vectors are
    \begin{align}
        b_{1} &= b^{(1)}, \nonumber \\
        b_{2} &= b^{(2)} - W_{2}b^{(1)}, \\
        b_{3} &= b^{(3)} - W_{3}b^{(2)}, \nonumber 
    \end{align}
    and $\underline{\theta}_{i} = (W_{i}, b_{i})$, $i=1, 2, 3$.
\end{proof}

\vspace{1em}
\noindent
{\bf Acknowledgments:}
The author thanks Thomas Chen for helpful comments, and J. Elisenda Grigsby for 
discussions and for pointing out the 
the works by Hanin and Sellke, and Johnson.
P.M.E. was supported by NSF grant DMS-2009800. 

\bibliographystyle{alpha} 

\end{document}

%% file: images/crescent3.pdf_tex
\begingroup%
  \makeatletter%
  \providecommand\color[2][]{%
    \errmessage{(Inkscape) Color is used for the text in Inkscape, but the package 'color.sty' is not loaded}%
    \renewcommand\color[2][]{}%
  }%
  \providecommand\transparent[1]{%
    \errmessage{(Inkscape) Transparency is used (non-zero) for the text in Inkscape, but the package 'transparent.sty' is not loaded}%
    \renewcommand\transparent[1]{}%
  }%
  \providecommand\rotatebox[2]{#2}%
  \newcommand*\fsize{\dimexpr\f@size pt\relax}%
  \newcommand*\lineheight[1]{\fontsize{\fsize}{#1\fsize}\selectfont}%
  \ifx\svgwidth\undefined%
    \setlength{\unitlength}{195.7015823bp}%
    \ifx\svgscale\undefined%
      \relax%
    \else%
      \setlength{\unitlength}{\unitlength * \real{\svgscale}}%
    \fi%
  \else%
    \setlength{\unitlength}{\svgwidth}%
  \fi%
  \global\let\svgwidth\undefined%
  \global\let\svgscale\undefined%
  \makeatother%
  \begin{picture}(1,0.52769888)%
    \lineheight{1}%
    \setlength\tabcolsep{0pt}%
    \put(0,0){\includegraphics[width=\unitlength,page=1]{crescent3.pdf}}%
    \put(0.16449436,0.23389598){\color[rgb]{0,0.37254902,0.5254902}\transparent{0.86111099}\makebox(0,0)[lt]{\lineheight{1.25}\smash{\begin{tabular}[t]{l}$\mathcal{X}_{2}$\end{tabular}}}}%
    \put(0.43565739,0.44330011){\color[rgb]{0,0.37254902,0.5254902}\transparent{0.86111099}\makebox(0,0)[lt]{\lineheight{1.25}\smash{\begin{tabular}[t]{l}$\mathcal{X}_{1}$\end{tabular}}}}%
    \put(0,0){\includegraphics[width=\unitlength,page=2]{crescent3.pdf}}%
    \put(0.62460212,0.28809135){\color[rgb]{0.74901961,0.34117647,0}\transparent{0.85632664}\makebox(0,0)[lt]{\lineheight{1.25}\smash{\begin{tabular}[t]{l}$v_1$\end{tabular}}}}%
    \put(0.54792672,0.14307061){\color[rgb]{0.74901961,0.34117647,0}\transparent{0.85632664}\makebox(0,0)[lt]{\lineheight{1.25}\smash{\begin{tabular}[t]{l}$v_2$\end{tabular}}}}%
    \put(0.4429463,0.20415164){\color[rgb]{0.74901961,0.34117647,0}\transparent{0.85632664}\makebox(0,0)[lt]{\lineheight{1.25}\smash{\begin{tabular}[t]{l}$p$\end{tabular}}}}%
  \end{picture}%
\endgroup%

%% file: images/concentric-copy.pdf_tex
\begingroup%
  \makeatletter%
  \providecommand\color[2][]{%
    \errmessage{(Inkscape) Color is used for the text in Inkscape, but the package 'color.sty' is not loaded}%
    \renewcommand\color[2][]{}%
  }%
  \providecommand\transparent[1]{%
    \errmessage{(Inkscape) Transparency is used (non-zero) for the text in Inkscape, but the package 'transparent.sty' is not loaded}%
    \renewcommand\transparent[1]{}%
  }%
  \providecommand\rotatebox[2]{#2}%
  \newcommand*\fsize{\dimexpr\f@size pt\relax}%
  \newcommand*\lineheight[1]{\fontsize{\fsize}{#1\fsize}\selectfont}%
  \ifx\svgwidth\undefined%
    \setlength{\unitlength}{117.72802158bp}%
    \ifx\svgscale\undefined%
      \relax%
    \else%
      \setlength{\unitlength}{\unitlength * \real{\svgscale}}%
    \fi%
  \else%
    \setlength{\unitlength}{\svgwidth}%
  \fi%
  \global\let\svgwidth\undefined%
  \global\let\svgscale\undefined%
  \makeatother%
  \begin{picture}(1,1.09766132)%
    \lineheight{1}%
    \setlength\tabcolsep{0pt}%
    \put(0,0){\includegraphics[width=\unitlength,page=1]{concentric-copy.pdf}}%
    \put(0.42902244,0.32372028){\color[rgb]{0,0.37254902,0.5254902}\transparent{0.86111099}\makebox(0,0)[lt]{\lineheight{1.25}\smash{\begin{tabular}[t]{l}$\mathcal{X}_{1}$\end{tabular}}}}%
    \put(0.47904197,0.07818199){\color[rgb]{0,0.37254902,0.5254902}\transparent{0.86111099}\makebox(0,0)[lt]{\lineheight{1.25}\smash{\begin{tabular}[t]{l}$\mathcal{X}_{2}$\end{tabular}}}}%
    \put(0,0){\includegraphics[width=\unitlength,page=2]{concentric-copy.pdf}}%
    \put(0.61156086,0.64862899){\color[rgb]{0.74901961,0.34117647,0}\makebox(0,0)[lt]{\lineheight{1.25}\smash{\begin{tabular}[t]{l}$p$\end{tabular}}}}%
    \put(0.76678191,0.84592919){\color[rgb]{0.74901961,0.34117647,0}\makebox(0,0)[lt]{\lineheight{1.25}\smash{\begin{tabular}[t]{l}$v_1$\end{tabular}}}}%
    \put(0.30359901,0.86650611){\color[rgb]{0.74901961,0.34117647,0}\makebox(0,0)[lt]{\lineheight{1.25}\smash{\begin{tabular}[t]{l}$v_2$\end{tabular}}}}%
    \put(0.52488261,0.80439505){\color[rgb]{0.74901961,0.34117647,0}\makebox(0,0)[lt]{\lineheight{1.25}\smash{\begin{tabular}[t]{l}$v_3$\end{tabular}}}}%
    \put(0,0){\includegraphics[width=\unitlength,page=3]{concentric-copy.pdf}}%
  \end{picture}%
\endgroup%

%% file: images/polyhedral.pdf_tex
\begingroup%
  \makeatletter%
  \providecommand\color[2][]{%
    \errmessage{(Inkscape) Color is used for the text in Inkscape, but the package 'color.sty' is not loaded}%
    \renewcommand\color[2][]{}%
  }%
  \providecommand\transparent[1]{%
    \errmessage{(Inkscape) Transparency is used (non-zero) for the text in Inkscape, but the package 'transparent.sty' is not loaded}%
    \renewcommand\transparent[1]{}%
  }%
  \providecommand\rotatebox[2]{#2}%
  \newcommand*\fsize{\dimexpr\f@size pt\relax}%
  \newcommand*\lineheight[1]{\fontsize{\fsize}{#1\fsize}\selectfont}%
  \ifx\svgwidth\undefined%
    \setlength{\unitlength}{136.87918115bp}%
    \ifx\svgscale\undefined%
      \relax%
    \else%
      \setlength{\unitlength}{\unitlength * \real{\svgscale}}%
    \fi%
  \else%
    \setlength{\unitlength}{\svgwidth}%
  \fi%
  \global\let\svgwidth\undefined%
  \global\let\svgscale\undefined%
  \makeatother%
  \begin{picture}(1,0.6571845)%
    \lineheight{1}%
    \setlength\tabcolsep{0pt}%
    \put(0,0){\includegraphics[width=\unitlength,page=1]{polyhedral.pdf}}%
    \put(0.44497674,0.20598759){\color[rgb]{0,0.37254902,0.5254902}\transparent{0.86111099}\makebox(0,0)[lt]{\lineheight{1.25}\smash{\begin{tabular}[t]{l}$p$\end{tabular}}}}%
    \put(0.49013618,0.4728422){\color[rgb]{0,0.37254902,0.5254902}\transparent{0.86111099}\makebox(0,0)[lt]{\lineheight{1.25}\smash{\begin{tabular}[t]{l}$v_1$\end{tabular}}}}%
    \put(0.1791734,0.19486865){\color[rgb]{0.74901961,0.34117647,0}\transparent{0.86111099}\makebox(0,0)[lt]{\lineheight{1.25}\smash{\begin{tabular}[t]{l}$\mathcal{S}_{-}$\end{tabular}}}}%
    \put(0.70558042,0.41151514){\color[rgb]{0.34117647,0.61568627,0.25882353}\transparent{0.86111099}\makebox(0,0)[lt]{\lineheight{1.25}\smash{\begin{tabular}[t]{l}$\mathcal{S}_{+}$\end{tabular}}}}%
    \put(0.28969986,0.01582972){\color[rgb]{0,0.37254902,0.5254902}\transparent{0.50991291}\makebox(0,0)[lt]{\lineheight{1.25}\smash{\begin{tabular}[t]{l}$\mathcal{S}$\end{tabular}}}}%
    \put(0.72859172,0.13213701){\color[rgb]{0,0.37254902,0.5254902}\transparent{0.86111099}\makebox(0,0)[lt]{\lineheight{1.25}\smash{\begin{tabular}[t]{l}$v_3$\end{tabular}}}}%
    \put(0.59382828,0.31460994){\color[rgb]{0,0.37254902,0.5254902}\transparent{0.86111099}\makebox(0,0)[lt]{\lineheight{1.25}\smash{\begin{tabular}[t]{l}$v_2$\end{tabular}}}}%
    \put(0,0){\includegraphics[width=\unitlength,page=2]{polyhedral.pdf}}%
  \end{picture}%
\endgroup%